\def\year{2021}\relax
\documentclass[letterpaper]{article} 
\usepackage{aaai21}  
\usepackage{times}  
\usepackage{helvet} 
\usepackage{courier}  
\usepackage[hyphens]{url}  
\usepackage{graphicx} 
\usepackage{natbib}  
\usepackage{caption} 
\usepackage{subfigure}
\usepackage[switch]{lineno}
\usepackage{subfig}
\usepackage{makecell}
\usepackage{bbding}
\usepackage{booktabs}
\usepackage{color}
\usepackage{amsmath}
\usepackage{amsfonts}

\newcommand{\ie}{\textit{i.e., }}
\newcommand{\eg}{\textit{e.g., }}

\DeclareMathOperator*{\argmin}{arg\,min}
\newtheorem{theorem}{Theorem}
\newtheorem{lemma}{Lemma}
\newtheorem{proof}{Proof}

\newtheorem{proposition}{Proposition}
\newtheorem{definition}{Definition}

\newcommand{\temp}[1]{{\color{black}#1}}

\title{Attentive Graph Neural Networks for Few-Shot Learning}
\author{
    Hao Cheng\textsuperscript{\rm 1},
    Joey Tianyi Zhou\textsuperscript{\rm 2},
    Wee Peng Tay\textsuperscript{\rm 1},
    Bihan Wen\textsuperscript{\rm 1}\thanks{Bihan Wen is the corresponding author.}
    \\
}
\affiliations {
    \textsuperscript{\rm 1} Nanyang Technological University \\
    \textsuperscript{\rm 2} Institute of High Performance Computing, Singapore \\
     \{hao006, wptay, bihan.wen\}@ntu.edu.sg, \quad joey.tianyi.zhou@gmail.com
}

\begin{document}
\maketitle
\begin{abstract}
Graph Neural Networks (GNN) has demonstrated the superior performance in many challenging applications, including the few-shot learning tasks.
Despite its powerful capacity to learn and generalize the model from few samples, GNN usually suffers from severe over-fitting and over-smoothing as the model becomes deep, which limit the scalability.
In this work, we propose a novel Attentive GNN to tackle these challenges, by incorporating a triple-attention mechanism, \ie node self-attention, neighborhood attention, and layer memory attention.
We explain why the proposed attentive modules can improve GNN for few-shot learning with theoretical analysis and illustrations.
Extensive experiments show that the proposed Attentive GNN model achieves the promising results, comparing to the state-of-the-art GNN- and CNN-based methods for few-shot learning tasks, over the mini-ImageNet and tiered-ImageNet benchmarks, under ConvNet-4 and ResNet-based backbone with both inductive and transductive settings.
The codes will be made publicly available.
\end{abstract}

\section{Introduction}
Deep neural networks, \eg Convolutional Neural Networks (CNNs), have been widely applied and achieved the superior results
on computer vision tasks such as image classification, segmentation, etc.
The conventional approach is by supervised learning over a large-scale labeled dataset for the task, thanks to the scalability of CNNs.
However, for some tasks with only a few samples available, training a highly-flexible deep model may result in over-fitting and thus fail to generalize.
Such a challenge presents in \textit{few-shot learning}~\cite{fei2006one}, in which a classifier is learned to predict the labels of the query samples using only a few labeled support samples of each class, as well as the training set contains only data of classes that are different from testing.
Various methods have been recently proposed for few-shot learning ~\cite{vinyals2016matching,snell2017prototypical,sung2018learning,garcia2017few}, including the popular \textit{meta-learning} framework~\cite{vinyals2016matching} based on \emph{episodic training}.
Meta-learning splits the training set into a large number of sub-tasks to simulate the testing task, which are used to train a meta-learner to adapt quickly to novel classes in few gradient updates.
Moreover, methods~\cite{vinyals2016matching,snell2017prototypical,sung2018learning} based on metric learning were proposed to learn a general feature embedding to exploit correlation between samples and classes.

Most of the existing few-shot learning methods are based on CNNs, which are effective at modeling image local properties.
However, in few-shot learning tasks, it is more important to exploit the intra- and inter-class relationships among samples.
Therefore, more recent works focused on learning the Graph Neural Network (GNN)~\cite{garcia2017few,liu2018learning,kim2019edge,yang2020dpgn} or Graph Convolutional Networks (GCN)~\cite{zhang2020adargcn,ye2020few}, to perform node or edge feature aggregation from neighbor samples by graph convolution.
However, several works~\cite{li2018deeper,rong2019dropedge} reported the over-fitting and over-smoothing issues when learning deeper GNN models (\ie poor scalability), as applying GCN or GNN is a special form of Laplacian smoothing, which averages the neighbors of the target nodes. 
Very recent work~\cite{rong2019dropedge} attempted to alleviate these obstacles via randomly dropping graph edges in training, showing improvement for node classification.
To the best of our knowledge, no work to date has addressed these issues for few-shot learning using graph attention mechanism.

In this work, we propose a novel Attentive GNN for highly-scalable and effective few-shot learning.
We incorporate a novel triple-attention mechanism, \ie node self-attention, neighborhood attention, and layer memory attention, to tackle the over-fitting and over-smoothing challenges towards more effective few-shot image classification.
Specifically, the node self-attention exploits inter-node and inter-class correlation beyond CNN-based features.
Neighborhood attention modules impose sparsity on the adjacency matrices, to attend to the most related neighbor nodes.
Layer memory attention applies dense connection to earlier-layer ``memory'' of node features.
Furthermore, we explain how the attentive modules help GNN generate discriminative features, and alleviate over-smoothing and over-fitting, with feature visualization and theoretical analysis.
We conduct extensive experiments showing that the proposed Attentive GNN achieves comparable results to the state-of-the-art methods over the mini-ImageNet and tiered-ImageNet datasets, under both inductive and transductive settings.

\section{Related Work}

\textbf{GNN for Few-shot Learning.} GNN~\cite{bruna2013spectral,defferrard2016convolutional} was first proposed for learning with the graph-structured data, and has been proved as a powerful technique for aggregating information from the neighboring vertices.
GNN was first used for few-shot learning~\cite{garcia2017few}, which aims to learn a complete graph network of nodes with both feature and class information.
Based on the episodic training mechanism, meta-graph parameters were trained to predict the label of a query node on the graph.
Later, TPN~\cite{liu2018learning} introduced the transductive setting into few-shot learning and constructed a top-k graph with a close-form label propagation based on node relationship.
Besides node label information, EGNN~\cite{kim2019edge} exploits edge information for directed graph by defining both class and edge labels for fully exploring the internal information of the graph.
Moreover, DPGN~\cite{yang2020dpgn} constructs a dual complete graph network to combine instance-level and distribution-level relations.

\textbf{Attention Mechanism.}
Attention Mechanism~\cite{vaswani2017attention} aims to focus more on regions which are more related for tasks and less on unrelated regions by learning a mask matrix or a weighted matrix.
In particular, self-attention~\cite{cheng2016long,parikh2016decomposable} considers the inherent correlation (attention) of the input features itself, which is mostly applied in graph model.
For node classification, GAT~\cite{velivckovic2017graph} used a graph attention layer to learn a weighted parameter vector based on entire neighborhoods to update node representation.
SAGPool~\cite{lee2019self} selected the top k$N$ nodes to generate mask matrix for graph pooling.
Moreover, attention mechanism is also utilized for few-shot learning.
CAN~\cite{hou2019cross} generated cross attention maps for each pair of nodes to highlight the object regions for classification.
Inspired by non-local block, Binary Attention Network~\cite{ke2020compare} considered a non-local attention module to learn the similarity globally.
Considering the attention between query sample with each support class, CTM~\cite{li2019finding} found task-relevant features based on both intra-class commonality and inter-class uniqueness. FEAT~\cite{ye2020few} utilized Transformer to learn task-specific adaptive instance embeddings.

\section{Attentive Graph Neural Networks}

\subsection{Preliminaries}
GNN~\cite{sperduti1997supervised,bruna2013spectral,defferrard2016convolutional} are the neural networks based on a graph structure $\bf{G}=(V,E)$ with nodes $\bf{V}$ and edges $\bf{E}$. Similar to the classic CNNs that exploit the local features (\eg image patch textures, sparsity) for representation, GNN is developed to mimic the behaviour of CNNs to deal with graph structured data which regards each sample (\eg image) as a vertex on the graph, and focuses on mining the important neighborhood information of each node, which is critical to construct discriminative and generalizable features for many tasks, \eg node classification, few-shot learning, etc. 
To be specific, considering a multi-stage GNN model, following~\cite{kipf2016semi} the output of the $k$-th GNN layer can be represented as:

\begin{equation} \label{eq1}
\mathbf{X}^{(k+1)} = \operatorname{F}_k ( \mathbf{X}^{(k)},\, \mathbf{W}^{(k)} ) = \rho \, ( \hat{\mathbf{A}}^{(k)} \, \mathbf{X}^{(k)} \, \mathbf{W}^{(k)} )
\end{equation}
where ${\mathbf{X}^{(k)}} \in \mathbb{R}^{V \times d_k}$ denotes the input feature and $x_{i}^{(k)}$ denotes the feature of node $i$ in the $k$-th layer, with $V$ and $d_{k}$ being the number of nodes and feature dimension at the $k$-th layer. 
Besides, $\hat{\mathbf{A}}^{(k)}  \in \mathbb{R}^{V \times V}$ is called the weighted adjacency matrix, $\mathbf{W}^{(k)} \in \mathbb{R}^{d_{k} \times d_{k+1}}$ is the trainable linear transformation, and $\rho$ denotes a non-linear function, \eg ReLU or Leaky-ReLU.

There are different ways to construct the adjacency matrix $\mathbf{A}^{(k)}$, \eg$\mathbf{A}^{(k)}_{i,j}$ indicates whether node $i$ and $j$ are directly connected in the classic GCN~\cite{bruna2013spectral}. 
Besides, $\mathbf{A}^{(k)}_{i,j}$ can be the similarity or distance between node $i$ and $j$~\cite{vinyals2016matching}, \ie $\mathbf{A}_{i,j}^{(k)}=f_\theta(\phi\left(x_i^{(k)}\right),\phi\left(x_j^{(k)}\right))$, where $\phi$ denotes the node feature embedding, and the parameter $\theta$ can be fixed or learned.
One popular example is to apply cosine correlation as the similarity metric, while a more flexible method is to learn a multi-layer perceptron (MLP) as the metric, \ie $f_\theta(x_i^{(k)},x_j^{(k)})=\mathrm{MLP}\left(\left|\mathrm{x}_{i}^{(k)}-\mathrm{x}_{j}^{(k)}\right|\right)$, where $\left|\cdot\right|$ denotes the absolute function.
More recent works applied Gaussian similarity function to construct the adjacency, \eg TPN~\cite{liu2018learning} proposed the similarity function as $\mathbf{A}_{i,j}=\exp \left(-0.5 d\left(\phi\left(x_i\right)/\sigma_{i}, \phi\left(x_j\right)\sigma_{j}\right)\right)$, 
with $\sigma$ being an example-wise length-scale parameter learned by a relation network of nodes used for normalization.

Different from the classic GNNs, the recently proposed GAT~\cite{velivckovic2017graph} exploited attention mechanism amongst all neighbor nodes in the feature domain, after the linear transformation $\mathbf{W}^{(k)}$ and computes the weights $\alpha$'s based on attention coefficients for graph update as:

\begin{equation} \label{eq:gat}
x_{i}^{(k+1)}=\rho\left(\sum \mathop{}_{\mkern-5mu j \in \mathcal{N}_{i}} \alpha_{i j} x_{j}^{(k)}\mathbf{W}^{(k)}\right)
\end{equation}
where $\mathcal{N}_{i}$ denotes the set of the neighbor (\ie connected) nodes of $x_{i}$.
GAT considers self-attention on the nodes after the linear transformation $\mathbf{W}$.
With a shared attention mechanism, GAT allows all neighbor nodes to attend on the target node.
However, GAT only considers the relationship among neighbors in the same layer while it fails to utilize the layer-wise information, which may lead to over-smoothing.

\subsection{What We Propose: Attentive GNN}
We propose an Attentive GNN model which contains three attentive mechanisms, \ie node self-attention, neighborhood attention, and layer memory attention. 
Fig.~\ref{fig:frame} shows the pipeline of Attentive GNN for the few-shot learning, and Fig.~\ref{fig:frame2} illustrates the details of one Attentive GNN layer.
We discuss each of the attention mechanisms, followed by how Attentive GNN is applied for few-shot learning.

\subsubsection{Node Self-Attention.}
Denote the feature of each sample $i$ (\ie node) as $\mathbf{x}_i \in \mathbb{R}^{d}$, and the one-hot vector of its corresponding label as $\mathbf{y}_{i}\in\mathbf{R}^{N}$, where $d$ is the feature dimension, $N$ is the total number of classes, and $1 \leq i \leq V$.
The one-hot vector sets only the element corresponding to the ground-truth category to be 1, while the others are all set to 0.
We propose the \textbf{node self-attention} to exploit the inter-class and inter-sample correlation at the initial stage. 
Denote the sample matrices and label matrices as:

\begin{equation}\begin{array}{c}
\mathbf{X}=\left[\mathbf{x}_{1}, \mathbf{x}_{2}, \ldots, \mathbf{x}_{V}\right]^{T} \in R^{V \times d}, \\
\mathbf{Y}=\left[\mathbf{y}_{1,} \mathbf{y}_{2,}, \ldots, \mathbf{y}_{V}\right]^{T} \in R^{V \times N}.
\end{array}\end{equation}
The first step is to calculate the sample and label correlation matrices as:

\begin{equation}
\mathbf{C}^{\mathbf{x}} = \mathrm{softmax}(\mathbf{X} \mathbf{X}^{T}\odot\mathbf{M}),\;\mathbf{C}^{\mathbf{y}} = \mathrm{softmax}(\mathbf{Y} \mathbf{Y}^T).
\end{equation}
Here, $\mathbf{M}$ is the normalization matrix defined as $\mathbf{M}(i,j)=\left(\left \| \mathbf{x}_i \right \|_2 \left \| \mathbf{x}_j \right \|_2\right)^{-1}$, $\odot$ denotes a point-wise product operator, and $\mathrm{softmax} (\cdot)$ denotes a row-wise softmax operator for the sample and label correlation matrices. 
Take the sample correlation as an example, and let $\mathbf{P} = \mathbf{X} \mathbf{X}^T \in \mathbb{R}^{V \times V}$. The row-wise softmax operator is defined as:

\begin{equation}
\mathbf{C}^{\mathbf{X}} (i, j) = \exp \left\{\mathbf{P} (i, j)\right\} / \sum \mathop{}_{\mkern-5mu k \in \mathcal{N}_{i}} \exp \left \{ \mathbf{P} (i, k)\right \}
\end{equation}
where $\mathcal{N}_{i}$ denotes the set of nodes that are connected to the $i$-th node.

The proposed node self-attention module exploits the correlation amongst both sample features and label vectors, which should share the information from different perspectives for the same node.
Thus, the next step is to fuse $\mathbf{C}^{\mathbf{x}}$ and $\mathbf{C}^{\mathbf{y}}$ using trainable $1 \times 1$ kernels as:

\begin{equation}\mathbf{C}^{\mathbf{f}}= \operatorname{fusion}_{\tau}\left(\left[\mathbf{C}^{\mathbf{x}},\mathbf{C}^{\mathbf{y}} \right]\right) \in \mathbf{R}^{V \times V}
\end{equation}
where $\left[\mathbf{C}^{\mathbf{x}},\mathbf{C}^{\mathbf{y}} \right]$ denotes the attention map concatenation, and $\tau$ is the kernel coefficients.
\temp{
This fusion function $\operatorname{fusion}_{\tau}$ is equivalent to $\mathbf{C}^f = w_1\mathbf{C}^x +w_2\mathbf{C}^y$, where the weighted parameters $w_1$, $w_2$ are learned adaptively.
}
With the fused self-attention map, both the feature and the label vectors are updated on the nodes:

\begin{equation} \label{eq:selfAtt}
\tilde{\mathbf{X}}^{(1)} = \mathbf{C}^{\mathbf{f}} \mathbf{X} \;, \;\;\;\; \mathbf{Y}^{(1)}=\alpha\mathbf{Y} + \left(1-\alpha\right)\mathbf{C}^{\mathbf{f}}\mathbf{Y} 
\end{equation}
where $\alpha\in[0,1]$ is a weighting parameter. 
Different from the feature update, the label update preserves the initial labels, which are the ground truth, in the support set, using the weighting parameter $\alpha$ to regularize the label update.
The updated sample features $\tilde{\mathbf{X}}^{(1)}$ and labels $\mathbf{Y}^{(1)}$ are concatenated to form the node features $\mathbf{X}^{(1)}$.
 \begin{figure*}[!t]
   \centering
   \includegraphics[width=0.75\paperwidth]{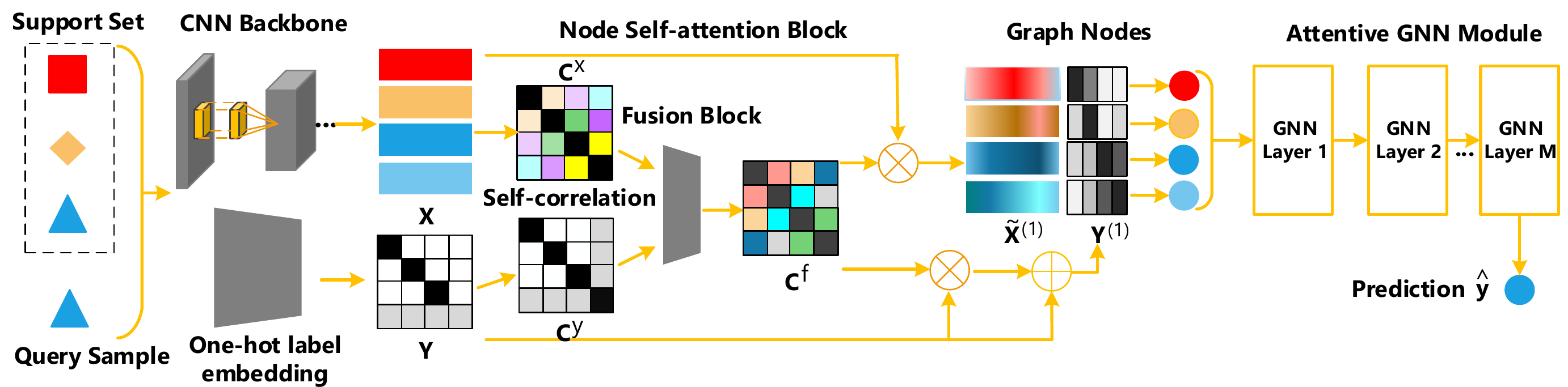}
   \caption{Illustration of the proposed Attentive GNN framework for the few-shot learning task. In the support set and query sample, the color and shape of the sample represent its corresponding class.}
   \label{fig:frame}
 \end{figure*}
  \begin{figure}[!t]
   \centering
   \includegraphics[width=1.0\linewidth]{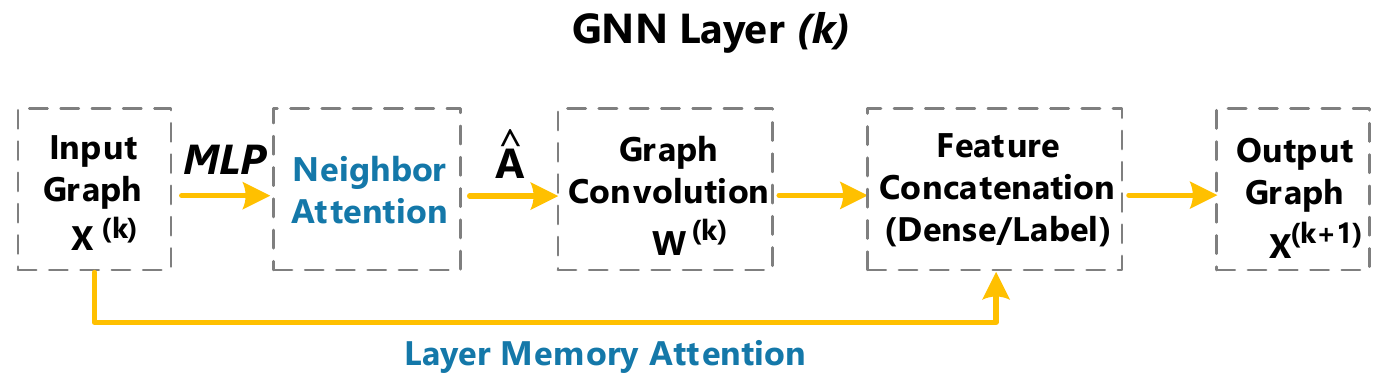}
   \caption{Illustration of $k$-th layer of Attentive GNN module.}
   \label{fig:frame2}
 \end{figure}
\subsubsection{Graph Neighbor Attention via Sparsity.}
Similar to various successful GNN framework, the proposed Attentive GNN applies a MLP to learn the adjacency matrix $\mathbf{A}_{ij}$ for feature updates.
When the GNN model becomes deeper, the risk of over-smoothing increases as GNN tends to mix information from all neighbor nodes and eventually converge to a stationary point in training.
To tackle this challenge, we propose a novel \textbf{graph neighbor attention} via sparsity constraint to attend to the most related nodes:

\begin{equation} \label{eq:attSpar}
\begin{array}{c}
\hat{\mathbf{A}}^{(k)} = \argmin_{\mathbf{A}^{(k)}} \left \| \mathbf{A}^{(k)} - \mathbf{U}^{(k)} \right \|_F \\
s.t. \; \mathbf{U}^{(k)} (i,j) = \mathrm{MLP}^{(k)}\left(\left|\mathbf{x}_{i}^{(k)}-\mathbf{x}_{j}^{(k)}\right|\right)\!, \left\| \mathbf{A}_i^{(k)}\right\|_0 \!\leq\! \beta V.
\end{array}
\end{equation}
Here, $\mathbf{A}_i^{(k)} \in \mathbb{R}^{1 \times V}$ denotes the $i$-th row of $\mathbf{A}^{(k)}$, $\beta \in (0, 1]$ denotes the ratio of nodes maintained for feature update, and $V$ is the number of graph nodes.
With the $\ell_0$ constraint, the adjacency matrix $\hat{\mathbf{A}}^{(k)}$ has up to $\beta V$ non-zeros in each row, corresponding to the \textbf{attended neighbor nodes}.
The solution to (\ref{eq:attSpar}) is achieved using the projection onto a $\ell_0$ unit ball, \ie keeping the $\beta V$ elements of each $\mathbf{U}_i^{(k)}$ with the largest magnitudes~\cite{wen2015structured}. Since the solution to (\ref{eq:attSpar}) is non-differentiable, we apply alternating projection for training, \ie in each epoch $\mathbf{U}^{(k)}$'s are first updated using back-propagation, followed by (\ref{eq:attSpar}) to update $\hat{\mathbf{A}}^{(k)}$ which is constrained to be sparse.
\temp{
For simplicity, we keep the top-k value for each row of $\mathbf{U}^{(k)}$ and set the others to 0 to construct the sparse matrix with $k=\beta V$.}

\subsubsection{Layer Memory Attention.}
To avoid the over-smoothing and over-fitting issues due to ``over-mixing'' neighboring nodes information, another approach is to attend to the ``earlier memory'' of intermediate features at previous layers.
Inspired by DenseNet~\cite{huang2017densely}, JKNet~\cite{xu2018representation}, GFCN~\cite{ji2020gfcn} and few-shot GNN~\cite{garcia2017few}, we densely connect the output of each GNN layer, as the intermediate GNN-node features maintain the consistent and more general representation across different GNN layers.

The proposed attentive GNN applies the transition function based on (\ref{eq1}).
In addition, we utilize graph self-loop \ie identity matrix $\mathbf{I}$ to utilize self information as:

\begin{equation} \label{neweq1}
\operatorname{F}_k ( \mathbf{X}^{(k)},\, \mathbf{W}^{(k)} ) = \rho \, ( \left[\hat{\mathbf{A}}^{(k)} \, \mathbf{X}^{(k)} \| \,\mathbf{I} \, \mathbf{X}^{(k)}\right] \mathbf{W}^{(k)} )
\end{equation}
where $\|$ means row-wise feature concatenation and $\mathbf{W}^{(k)} \in \mathbb{R}^{2d_k \times m}$.
Furthermore, instead of using $\operatorname{F}_k (\mathbf{X}^{(k)},\mathbf{W}^{(k)}) \in \mathbf{R}^{V \times m}$ directly as the input node feature at the $(k$+$1)$-th layer, we propose to attend to the ``early memory'' referring to~\cite{garcia2017few} by concatenating the node feature at the $k$-th layer as:

\begin{equation} \label{gnn_dense}
\mathbf{X}^{(k+1)} = \left[ \mathbf{X}^{(k)}, \operatorname{F}_{k} \left( \mathbf{X}^{(k)}, \mathbf{W}^{(k)} \right) \right] \in \mathbb{R}^{{V} \times {(d + N + km)} }\;.
\end{equation}

However, as the number of GNN layers $k$ increases, the time and memory complexity will continue to increase. 
In order to reduce the memory complexity, we use label feature concatenation to replace dense connection as:

\begin{equation} \label{gnn_labelcon}
\mathbf{X}^{(k+1)} = \left[ \operatorname{F}_{k} \left( \mathbf{X}^{(k)}, \mathbf{W}^{(k)} \right), \mathbf{Y}^{(1)} \right] \in \mathbb{R}^{{V} \times {(m + N)} }\;.
\end{equation}
where $\mathbf{Y}^{(1)}$ is the new label feature updated by (\ref{eq:selfAtt}).

For these two memory attention mechanisms, there is only 
$V \times m$ new features introduced in a new layer, while the other features (node feature of earlier layer $\mathbf{X}^{(k)}$ or label feature $\mathbf{Y}^{(1)}$ ) are attended to the early memory.

\subsection{Application: Few-Shot Learning Task}

\subsubsection{Problem Definition.}
We apply the proposed Attentive GNN for the few-shot image classification tasks: 
Given a large-scale and labeled training set with classes $\mathcal{C}_{train}$, and a few-shot testing set with classes $\mathcal{C}_{test}$ which are mutually exclusive, \ie $\mathcal{C}_{train} \cap \mathcal{C}_{test} = \emptyset$, we aim to train a classification model over the training set, which could be applied to the test set with only few labeled information.
Such test is called the $N$-way $K$-shot task $\mathcal{T}$, where $K$ is the number of labeled samples which is often set from 1 to 5, \ie the testing set contains a support set $\mathcal{S}$ which is labeled, and a query set $\mathcal{Q}$ to be predicted, denoted as $\mathcal{T}=\mathcal{S} \cup \mathcal{Q}$.
The $N$ and $K$ are both very small for few-shot learning.

\subsubsection{Attentive Model for Few-Shot Learning.}
Following the same strategy of episodic training~\cite{vinyals2016matching} and meta-learning framework, we simulate $N$-way $K$-shot tasks which are randomly sampled from the training set, in which the support set includes $K$ labeled samples (\eg images) from each of the $N$ classes and the query set includes unlabeled samples from the same $N$ classes. 
Each task is modeled as a complete graph~\cite{garcia2017few}, in which each node represents an image sample and its label.
The objective is to learn the parameters of Attentive GNN using the simulated tasks, which is generalizable for an unseen few-shot task.

\subsubsection{Loss Function.}
For each simulated few-shot task $\mathcal{T}_{train}$ with its query set $\mathcal{Q}=\left\{\left(\mathbf{x}_{i}, y_{i}\right)\right\}_{i=1}^{Q}$, the parameters of the backbone feature extractor, self-attention block $\tau$, and the GNNs $\{ \mathrm{MLP}^{(k)},\mathbf{W}^{(k)} \}$ are trained by minimizing the cross-entropy loss of classes over all query samples as:

\begin{equation}
\sum \mathop{}_{\mkern-5mu k \in \mathcal{N}_{i}}
\mathbf{L}_{\mathcal{T}_{train}}=-\sum \mathop{}^{\mkern-5mu Q}_{\mkern-5mu i = 1} y_{i} \log P\left(\hat{y}_i=y_{i} | \mathcal{T}\right)
\end{equation}
where $\hat{y}_i$ and $y_i$ denote the predicted and ground-truth labels of the query sample $\mathbf{x}_{i}$, respectively.
We evaluate the proposed Attentive GNN for few-shot task using both \textbf{inductive} and \textbf{transductive} settings, which correspond to $Q = 1$, and $Q = Nq$ with $q > 1$, respectively.
\temp{
And for each query sample in N-way K-shot task, we initialize the one-hot feature $y$ with uniform distribution (\ie each value is set to $1/N$) or all 0 (\ie each value is set to $0$).}

\section{Why it works}

\textbf{Discriminative Sample Representation.} It is critical to obtain the initial feature representation of the samples that are sufficiently discriminative (\ie samples of different classes are separated) for the GNN models in few-shot tasks.
However, most of the existing GNN models work with generic features using CNN-based backbone, and fail to capture the task-specific structure.
The proposed node self-attention module exploits the cross-sample correlation, and can thus effectively guide the feature representation for each few-shot task.
Fig.~\ref{fig-2} compares two examples of the graph features for 5-way 1-shot transductive learning using t-SNE visualization~\cite{maaten2008visualizing}, using the vanilla GNN and the proposed Attentive GNN.
The vanilla GNN generates node representation that are ``over-smoothed'' due to bad initial feature using CNN-based backbone.
On the contrary, the node self-attention module can effectively generate the discriminative features, which lead to the more promising results using the Attentive GNN.

\begin{figure*}[!t]
\centering
\subfigure[Vanilla GNN]{
    \includegraphics[width=5cm,height=3cm]{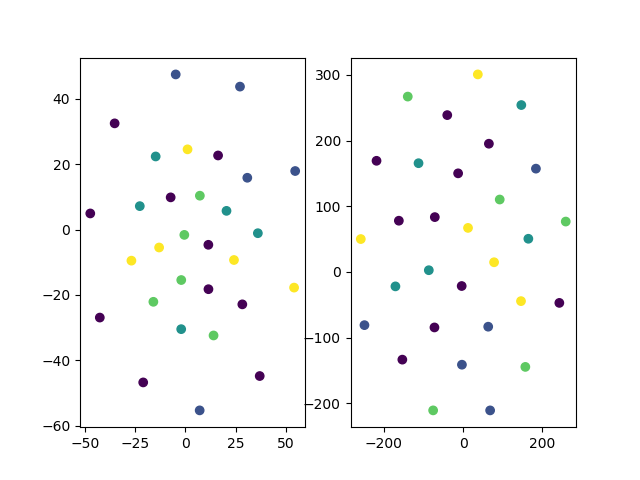}}
\quad
\subfigure[Attentive GNN]{
    \includegraphics[width=7cm,height=3cm]{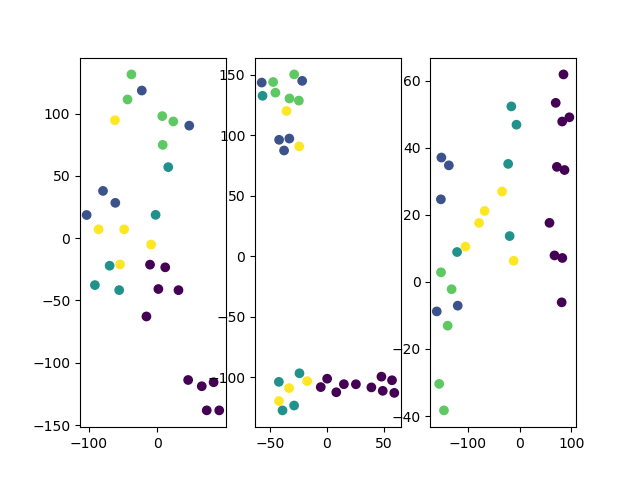}}
\caption{ t-SNE visualization~\cite{maaten2008visualizing} of the graph features under 5-way 1-shot transductive setting using (a) vanilla GNN, and (b) Attentive GNN. Samples of different classes are color-coded. Leftmost plots: the initial feature embedding; Rightmost plots: final output; Middle plot of (b): output by the node self-attention.}
\label{fig-2}
\end{figure*}
\textbf{Alleviation of Over-Smoothing and Over-Fitting problems.}
\textit{Over-fitting} arises when learning an over-parametric model from the limited training data, and it is extremely severe as the objective of few-shot learning is to generalize the knowledge from training set for few-shot tasks.
On the other hand, \textit{over-smoothing} phenomenon refers to the case where the features of all (connected) nodes converge to similar values as the model depth increases.
We provide theoretical analysis to show that the proposed triple-attention mechanism can alleviate both over-fitting and over-smoothing in GNN training.
For each of the result, the proof sketch is presented, while the corresponding \textbf{full proofs} are included in the \textbf{Appendix}.

\begin{lemma} \label{lemma:self-att}
The node self-attention module is equivalent to a GNN layer if $\alpha = 0$ as
\begin{equation}
    \mathbf{X}^{(k)} = \begin{bmatrix} \mathbf{X}, \mathbf{Y} \end{bmatrix}\,,\;\; \mathbf{A}^{(k)} = \mathbf{C}^{\mathbf{f}}\,,\;\; \mathbf{W}^{(k)} = \mathbf{I}\,.\;\;
\end{equation}
\end{lemma}
\noindent\textbf{Proof Sketch}
\textit{The feature and label vector updates using (\ref{eq:selfAtt}) is similar to multiplying with an adjacency matrix in (\ref{eq1}), while such matrix is obtained in a self-supervised way.}

\begin{proposition} \label{prop:self-att}
Applying the node self-attention module to replace a GNN layer in Attentive GNN, reduces the trainable-parameter complexity from $\mathcal{O}\{ d ( d + L) \}$ to $\mathcal{O}\{ 1 \}$, where $L$ denotes the depth of MLP for generating the adjacency metric.
\end{proposition}
\noindent\textbf{Proof Sketch}
\textit{The trainable parameters in a GNN layer (\ref{eq1}) are mainly the linear transformation $\mathbf{W}^{(k)}$ and the MLP$^{(k)}$, which scale as $\mathcal{O}\{d^2\}$ and $\mathcal{O}\{dL\}$, respectively. 
On the contrary, the graph self-attention only involve the $1 \times 1$ kernels that are trainable.}

Lemma~\ref{lemma:self-att} and Proposition~\ref{prop:self-att} prove that the node self-attention module involves much fewer trainable parameters than a normal GNN layer.
Thus, applying node self-attention instead of another GNN layer will reduce the model complexity, thus lowering the risk of over-fitting.

Next we show that using graph neighbor attention can help alleviate over-smoothing for training GNN. 
The analysis is based on the recent works on dropEdge~\cite{rong2019dropedge} and GNN information loss~\cite{oono2020graph}.
They proved that a sufficiently deep GNN model will always suffer from ``$\epsilon$-smoothing''~\cite{oono2020graph}, where $\epsilon$ is defined as the error bound of the maximum distance among node features. Another concept is the ``information loss''~\cite{oono2020graph} of a graph model $\bf{G}$, \ie the dimensionality reduction of the node feature-space after $T$ layers of GNNs, denoted as $\Theta_{T, \bf{G}}$. We use these two concepts to quantify the over-smoothing issue in our analysis.

\begin{theorem} \label{theo} 
Denote the same multi-layer GNN model with and without neighbor attention as $\tilde{\bf{G}}$ and $\bf{G}$, respectively.
Besides, denote the number of GNN layers for them to encounter the $\epsilon$-smoothing~\cite{oono2020graph} as $T(\tilde{\bf{G}}, \epsilon)$ and $T(\bf{G}, \epsilon)$, respectively.
With sufficiently small $\beta$ in the node self-attention module, either (i) $T(\tilde{\bf{G}}, \epsilon) \leq T(\bf{G}, \epsilon)$, or (ii) $\Theta_{T(\bf{G}, \epsilon), \bf{G}} > \Theta_{T(\tilde{\bf{G}}, \epsilon), \tilde{\bf{G}}}$, will hold.
\end{theorem}
\noindent\textbf{Remarks}
\textit{The result shows that the GNN model with graph neighbor attention (i) increases the maximum number of layers to encounter over-smoothing, or if the number of layers remains, (ii) the over-smoothing phenomenon is alleviated.}

\section{Experiments}
\textbf{Datasets}. We conducted extensive experiments to evaluate the effectiveness of the proposed Attentive GNN model for few-shot learning, over two widely-used few-shot image classification benchmarks, \ie mini-ImageNet~\cite{vinyals2016matching} and tiered-ImageNet~\cite{ren2018meta}.
Mini-ImageNet contains around 60000 images of 100 different classes extracted from the ILSVRC-12 challenge~\cite{krizhevsky2012imagenet}.
We used the proposed splits by~\cite{ravi2016optimization}, \ie 64, 16 and 20 classes for training, validation and testing, respectively.
Tiered-ImageNet dataset is a more challenging data subset from the ILSVRC-12 challenge~\cite{krizhevsky2012imagenet}, which contains more classes that are organized in a hierarchical structure, \ie 608 classes from 34 top categories.
We follow the setups proposed by~\cite{ren2018meta}, and split 34 top categories into 20 (resp. 351 classes), 6 (resp. 97 classes), and 8 (resp. 160 classes), for training, validation, and testing, respectively.
For both datasets, all images are resized to $84 \times 84$. 
\begin{table*}[!t]
   \centering
   \resizebox{0.7\textwidth}{!}{
   \begin{tabular}{llllll}
     \toprule
     \multicolumn{2}{c}{}  &     \multicolumn{2}{c}{mini-ImageNet}  & \multicolumn{2}{c}{tiered-ImageNet}  \\
     \cmidrule(r){3-4} \cmidrule(r){5-6}
     Model     & Trans     &  5-way 1-shot & 5-way 5-shot  & 5-way 1-shot & 5-way 5-shot \\
     \midrule
     Proto-Net~\cite{snell2017prototypical}   & \makecell[c]{\XSolidBrush}   & \makecell[c]{46.14}  & \makecell[c]{65.77} & \makecell[c]{48.58} & \makecell[c]{69.57}\\
     Matching-Net ~\cite{vinyals2016matching}     & \makecell[c]{\XSolidBrush}   & \makecell[c]{46.60}  & \makecell[c]{55.30} & \makecell[c]{-} & \makecell[c]{-}\\
     Reptile~\cite{nichol2018first}                & \makecell[c]{\XSolidBrush}   & \makecell[c]{47.07}  & \makecell[c]{62.74} & \makecell[c]{48.97} &\makecell[c]{66.47} \\
     Relation-Net~\cite{sung2018learning}       & \makecell[c]{\XSolidBrush}   & \makecell[c]{50.44}  & \makecell[c]{65.32} & \makecell[c]{-} & \makecell[c]{-}\\
     MRN-Zero~\cite{he2020memory}       & \makecell[c]{\XSolidBrush}   & \makecell[c]{51.62}  & \makecell[c]{67.06} & \makecell[c]{\bf{\textcolor{blue}{55.22}}} & \makecell[c]{\bf{\textcolor{red}{73.37}}}\\
     SAML~\cite{hao2019collect}                   & \makecell[c]{\XSolidBrush}   & \makecell[c]{52.22}  & \makecell[c]{66.49} & \makecell[c]{-} & \makecell[c]{-}\\
     E$^3$BM~\cite{Liu2020E3BM}   & \makecell[c]{\XSolidBrush}   &  \makecell[c]{53.20}   & \makecell[c]{65.10} & \makecell[c]{52.10} & \makecell[c]{70.20}\\
     FEAT~\cite{ye2020few}   & 
     \makecell[c]{\XSolidBrush}   &  \makecell[c]{\bf{\textcolor{red}{55.15}}}   & \makecell[c]{\bf{\textcolor{red}{71.61}}} & \makecell[c]{-} & \makecell[c]{-}\\
     \midrule
     \bf{GNN based methods}\\
     GNN~\cite{garcia2017few}                    & \makecell[c]{\XSolidBrush}   & \makecell[c]{50.33}  & \makecell[c]{66.41} & \makecell[c]{54.97} & \makecell[c]{70.92}\\
     EGNN~\cite{kim2019edge}                   & \makecell[c]{\XSolidBrush}   & \makecell[c]{52.86}  & 
     \makecell[c]{66.85} & \makecell[c]{-} & \makecell[c]{70.98}\\
     \bf{Ours}   & \makecell[c]{\XSolidBrush}   &  \makecell[c]{\bf{\textcolor{blue}{54.81}}}   & \makecell[c]{\bf{\textcolor{blue}{69.85}}} & \makecell[c]{\bf{\textcolor{red}{57.47}}} & \makecell[c]{\bf{\textcolor{blue}{72.29}}}\\
     \midrule
     MAML~\cite{finn2017model}                   & \makecell[c]{BN}    & \makecell[c]{48.70}  & \makecell[c]{63.11} & \makecell[c]{51.67} & \makecell[c]{70.30}\\
     Reptile~\cite{nichol2018first}                & \makecell[c]{BN}    & \makecell[c]{49.97}  & \makecell[c]{65.99} & \makecell[c]{52.36} & \makecell[c]{71.03}\\
     MAML~\cite{finn2017model}                   & \makecell[c]{\Checkmark}   & \makecell[c]{50.83}  & \makecell[c]{66.19} & \makecell[c]{53.23} & \makecell[c]{70.83}\\
     DN4~\cite{li2019revisiting}                    & \makecell[c]{\Checkmark}   & \makecell[c]{51.24}  & \makecell[c]{71.02} & \makecell[c]{-} & \makecell[c]{-}\\
     Relation-Net~\cite{sung2018learning}       & \makecell[c]{\Checkmark}   & \makecell[c]{51.38}  & \makecell[c]{67.07} & \makecell[c]{54.48} & \makecell[c]{71.31}\\
     Dynamic-Net~\cite{gidaris2018dynamic}       & \makecell[c]{\Checkmark}   & \makecell[c]{56.20}  & \makecell[c]{72.81} & \makecell[c]{-} & \makecell[c]{-}\\
     TEAM~\cite{qiao2019transductive}       & \makecell[c]{\Checkmark}   & \makecell[c]{56.57}  & \makecell[c]{72.04} & \makecell[c]{-} & \makecell[c]{-}\\
     FEAT~\cite{ye2020few}   & 
     \makecell[c]{\Checkmark}   &  \makecell[c]{57.04}   & \makecell[c]{72.89} & \makecell[c]{-} & \makecell[c]{-}\\
     MRN~\cite{he2020memory}   & 
     \makecell[c]{\Checkmark}   &  \makecell[c]{57.83}   & \makecell[c]{71.13} & \makecell[c]{62.65} & \makecell[c]{74.20}\\
     \midrule
     \bf{GNN based methods}\\
     TPN$^{\sharp}$~\cite{liu2018learning}         & \makecell[c]{\Checkmark}   & \makecell[c]{51.94}  & \makecell[c]{67.55} & \makecell[c]{59.91} & \makecell[c]{73.30}\\
     TPN~\cite{liu2018learning}                    & \makecell[c]{\Checkmark}   & \makecell[c]{53.75}  & \makecell[c]{69.43} & \makecell[c]{57.53} & \makecell[c]{72.85}\\
     GNN$^{\natural}$~\cite{garcia2017few}       & \makecell[c]{\Checkmark}   & \makecell[c]{54.14}  & \makecell[c]{70.38} & \makecell[c]{65.11} & \makecell[c]{76.40}\\
     EGNN~\cite{kim2019edge}                   & \makecell[c]{\Checkmark}   & \makecell[c]{59.18}  & \makecell[c]{\bf{\textcolor{red}{76.37}}} & \makecell[c]{63.52} & \makecell[c]{\bf{\textcolor{red}{80.15}}}\\
     \bf{Ours (Normalized)}    & \makecell[c]{\Checkmark}   &
     \makecell[c]{\bf{\textcolor{blue}{59.87}}} & 
     \makecell[c]{\bf{\textcolor{blue}{74.46}}} & \makecell[c]{\bf{\textcolor{blue}{66.87}}} & \makecell[c]{79.26}\\
     \bf{Ours}    & \makecell[c]{\Checkmark}   & \makecell[c]{\bf{\textcolor{red}{60.14}}}  & \makecell[c]{\bf{72.41}} & \makecell[c]{\bf{\textcolor{red}{67.23}}} & \makecell[c]{\bf{\textcolor{blue}{79.55}}}\\
     \bottomrule
   \end{tabular}
   }
    \caption{Few-shot classification accuracy averaged over mini-ImageNet and tiered-ImageNet with the ConvNet-4 backbone. The best and second best results under each setting and dataset are highlighted as \textcolor{red}{Red} and \textcolor{blue}{Blue}, respectively. ``BN'' denotes that the batch normalization where query statistical information is used instead of global batch normalization.}
     \label{table1}
 \end{table*}
\begin{table}[!t]
   \centering
   \resizebox{1\columnwidth}{!}{
   \begin{tabular}{lllll}
     \toprule
     \multicolumn{2}{c}{}  & \multicolumn{2}{c}{tiered-ImageNet}  \\
     \cmidrule(r){3-4}
     Model   & Backbone  &  5-way 1-shot & 5-way 5-shot \\
     \midrule
     Meta-Transfer~\cite{sun2019meta}  & ResNet12  & \makecell[c]{65.62$\pm$ 1.80} & \makecell[c]{80.61$\pm$ 0.90}\\
     Proto-Net~\cite{snell2017prototypical}  & ResNet12  & \makecell[c]{65.65$\pm$ 0.92} & \makecell[c]{83.40$\pm$ 0.65}\\
     MetaOptNet~\cite{lee2019meta}  & ResNet12  & \makecell[c]{65.99$\pm$ 0.72} & \makecell[c]{81.56$\pm$ 0.53}\\     
     CTM~\cite{li2019finding} & ResNet18    & \makecell[c]{68.41$\pm$ 0.39} & \makecell[c]{84.28$\pm$ 1.73}\\
     Meta-Baseline~\cite{chen2020new}  & ResNet12  & \makecell[c]{68.62$\pm$ 0.27} & \makecell[c]{83.29$\pm$ 0.18}\\
     CAN~\cite{hou2019cross} & ResNet12    & \makecell[c]{69.89$\pm$ 0.51} & \makecell[c]{84.23$\pm$ 0.37}\\
     FEAT~\cite{ye2020few} & ResNet12  & \makecell[c]{70.80$\pm$ 0.23} & \makecell[c]{84.79$\pm$ 0.16}\\
     DeepEMD~\cite{Zhang_2020_CVPR} & ResNet12    & \makecell[c]{71.16$\pm$ 0.87} & \makecell[c]{\bf{\textcolor{blue}{86.03$\pm$ 0.58}}}\\
     CAN+Trans~\cite{hou2019cross} & ResNet12    & \makecell[c]{\bf{\textcolor{red}{73.21$\pm$ 0.58}}} & \makecell[c]{84.93$\pm$ 0.38}\\
     \midrule
     \bf{GNN based methods}\\
     TPN~\cite{liu2018learning}  & ResNet12  & \makecell[c]{59.91$\pm$ 0.94} & \makecell[c]{73.30$\pm$ 0.75}\\
     DPGN~\cite{yang2020dpgn} & ResNet12    & \makecell[c]{72.45$\pm$ 0.51} & \makecell[c]{\bf{\textcolor{red}{87.24$\pm$ 0.39}}}\\
     \bf{Ours}  & ResNet12   &  \makecell[c]{\bf{\textcolor{blue}{73.15$\pm$ 0.27}}} & \makecell[c]{84.96$\pm$ 0.12}\\
     \bottomrule
   \end{tabular}
   }
   \caption{Few-shot classification accuracy averaged over tiered-ImageNet with the ResNet backbone. The best (second best resp.) results are highlighted as \textcolor{red}{Red}  (\textcolor{blue}{Blue} resp.).}
   \label{res}
 \end{table}

\textbf{Implementation Details}. We follow most of the DNN-based few-shot learning schemes~\cite{snell2017prototypical,garcia2017few} 
and first apply the popular ConvNet-4 as the backbone feature extractor, with $3\times 3$ convolution kernels, numbers of channels as $[64,96,128,256]$ at corresponding layers, a batch normalization layer, a max pooling layer, and a LeakyReLU activation layer. 
Besides, two dropout layers are adapted to the last two convolution blocks to alleviate over-fitting~\cite{garcia2017few}.
Furthermore, to compare with the more complicated CNN-based methods, we also apply ResNet-12 as the backbone, following the similar setup in~\cite{oreshkin2018tadam}.
On this basis, a fully-connected layer with batch normalization is added to the end for dimensionality reduction.
We conducted both 5-way 1-shot, and 5-way 5-shot experiments, under both inductive and transductive settings~\cite{liu2018learning}.
We use only one query sample for the inductive, and $5$ ($15$ resp.) query samples per class for the transductive experiments on ConvNet-4 (ResNet12 resp.) backbone.
Our models are all trained using Adam optimizer with an initial learning rate of $1\times10^{-3}$.
For ConvNet-4 backbone, the weight decay is set to $10^{-6}$ and the mini-batch sizes are set to 100 / 40 and 30 / 20, for 1-shot / 5-shot inductive and transductive settings, respectively.
We reduce the learning rate to half every 15K and 30K epochs, over mini-ImageNet and tiered-ImageNet, respectively.
For ResNet12 backbone, the weight decay is $10^{-5}$ and the mini-batch sizes are set to 28.
We cut the learning rate to 0.1 every 15K epochs.
The output feature dimension of two backbones is 128 and the number of GNN layers is set to 3.

\textbf{Results}. 
We compare the proposed Attentive GNN to the state-of-the-art CNN- and GNN-based methods, using the ConvNet-4 and ResNet backbone, and Table~\ref{table1} and Table~\ref{res} list the average accuracy of the few-shot image classification, respectively.
For the same algorithm, the accuracy of transductive learning is typically better than that of the inductive learning, by further exploiting the correlation amongst the multiple query samples.
Different from the standard transductive setting which applies $5$ query samples per class, EGNN~\cite{kim2019edge} uses $15$ query samples which is more ``advantageous''. 
The proposed Attentive GNN has achieved promising results under different settings, comparing to the various state-of-the-art few-shot learning methods.

\textbf{Ablation Study}.
We investigate the effectiveness of each proposed attention module by conducting an ablation study.
Fig.~\ref{fig:ablation} plots the image classification accuracy over the tiered-ImageNet dataset, with different variations of the proposed Attentive GNN, by removing the node self-attention (self att) and layer memory attention (memory att) modules. 
Besides, instead of applying layer memory attention which attends to the concatenated early feature, we try another variation by concatenating only the label vectors (label concat).
It is clear that all variations generate degraded results, and some even suffer from more severe over-smoothing, \ie accuracy drops quickly as the number of GNN layers increases. 
We show that the label concatenation is a reasonable alternative to replace layer memory attention which requires less memory complexity.
Furthermore, we study the influence of the graph neighbor attention for few-shot learning by varying the hyper-parameter $\beta$.
Fig.~\ref{fig:beta} plots the inductive image classification accuracy by applying Attentive GNN with varying $\beta$ (\ie ratio of elements maintained in $\mathbf{A}$) in the graph neighbor attention module.
When $\beta = 1$, it is equivalent to removing the graph neighbor attention at all. 
By choosing the optimal $\beta \in [0, 1]$'s for 5-way 1-shot and 5-way 5-shot settings, respectively, the graph neighbor attention can further boost the classification results.
\begin{figure}[!t]
	\center
    \includegraphics[width=0.85\linewidth]{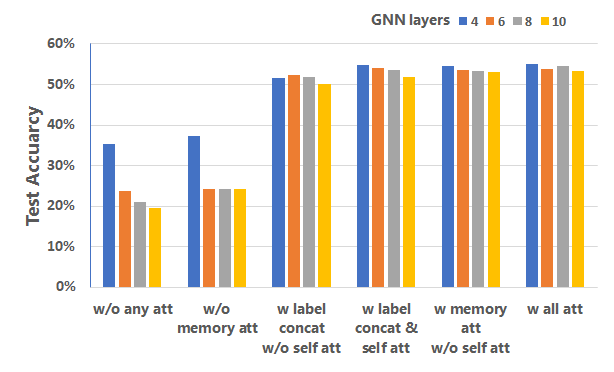}
	\caption{Ablation study: Classification accuracies using Attentive GNN (4 to 10 layers) and its variations, over tiered-ImageNet.}
	\label{fig:ablation}
\end{figure}
\begin{figure}[!t]
	\center
    \includegraphics[width=0.85\linewidth]{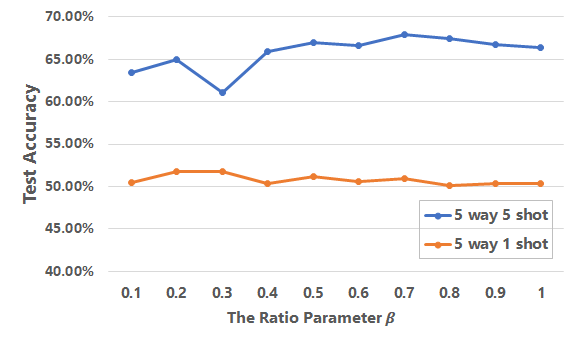}
	\caption{Inductive test accuracy vs. $\beta$ over mini-ImageNet.}
	\label{fig:beta}
\end{figure}

\textbf{Hyper-Parameters}.
There are two hyper-parameters in the proposed Attentive GNN, namely $\alpha$ and $\beta$, corresponding to the ratio for label fusion, and the sparsity ratio in the neighbor attention module. 
Table~\ref{table3} shows how varying these two parameters affects the inductive learning for image classification averaged over tiered-ImageNet.
Both $\alpha$ and $\beta$ range between 0 and 1. Besides, we also test the model when the label fusion mechanism is totally removed, denoted as ``-'' in the table.
The empirical results demonstrate the effectiveness of label fusion with $\alpha = 0.5$ to be a reasonable ratio.
Besides, for 5-way 1-shot learning, the best result is generated when $\beta = 1$, which is equivalent to remove the graph neighbor attention. 
It is because the total number of nodes is small for 5-way 1-shot learning, thus imposing sparsity leads to too restrictive model.
\begin{table}[!]
    \centering
    \resizebox{.4\textwidth}{!}{
   \begin{tabular}{llll}
     \toprule
      \multicolumn{2}{c}{Hyper-Parameter Setting}  &     \multicolumn{2}{c}{Accuracy}   \\
     \cmidrule(r){1-2} \cmidrule(r){3-4}
      \makecell[c]{$\beta$} &  \makecell[c]{$\alpha$} & 
      5-way 1-shot & 5-way 5-shot \\
     \midrule
     \makecell[c]{1.0} & \makecell[c]{-}  & \makecell[c]{54.97}   &  \makecell[c]{70.92}  \\
     \makecell[c]{0.7} & \makecell[c]{-}  & \makecell[c]{57.18}   &  \makecell[c]{70.58}  \\
     \makecell[c]{1.0} & \makecell[c]{0}    & \makecell[c]{57.41}   &  \makecell[c]{72.03}  \\
     \makecell[c]{1.0} & \makecell[c]{0.5}  & \makecell[c]{\bf{57.68}}   &  \makecell[c]{71.03}  \\
     \makecell[c]{0.7} & \makecell[c]{0.5}  & \makecell[c]{57.47}   &  \makecell[c]{\bf{72.29}}  \\
     \bottomrule
   \end{tabular}}
   \captionof{table}{{Inductive accuracy on tiered-ImageNet with different settings. ``-'' means NOT applying node self-attention.}}\label{table3}
\end{table}

\textbf{Robustness in Transductive Learning}.
While the query samples are always \textbf{uniformly} distributed for each class in the conventional transductive learning setting~\cite{liu2018learning}, such assumption may not hold in practice, \eg query set contains samples with \textbf{random} labels. 
We study how robust the proposed Attentive GNN is for such setting by comparing to the baseline GNN method~\cite{garcia2017few} and GNN with only neighbor attention (\eg w/ Neighbor Att.).
In the training, we simulate the query set with samples with random labels correspondingly for Attentive GNN and all competing methods under such setting.
Table~\ref{table4} shows the image classification accuracy with 5-way 1-shot transductive learning, averaged over tiered-ImageNet.
With the query-set samples of ``random'' labels, the proposed Attentive GNN can still generate significantly better results comparing to the vanilla GNN.
Table~\ref{table4} shows that the proposed graph neighbor attention module contributes to the robustness. As the sparse adjacency matrix can attend to the related nodes (\ie nodes with the same class) in an adaptive way, preventing ``over-mixing'' with all nodes.
\begin{table}[!]
    \centering
    \resizebox{0.4\textwidth}{!}{
     \begin{tabular}{lll}
     \toprule
     \makecell[c]{Model} &  \makecell[c]{Random} & \makecell[c]{Uniform} \\
     \midrule
     \makecell[c]{Vanilla GNN~\cite{garcia2017few}}  & \makecell[c]{59.77}   &  \makecell[c]{65.11}  \\
     \makecell[c]{Vanilla GNN w/ Neighbor Att.}  & \makecell[c]{60.18}   &  \makecell[c]{65.49}  \\
     \makecell[c]{Attentive GNN}    & \makecell[c]{\bf{61.39}}   &  \makecell[c]{\bf{67.23}}  \\
     \bottomrule
   \end{tabular} 
   }
   \captionof{table}{Effect of query samples distribution on tiered-ImageNet for 5-way 1-shot task under transductive setting. The total number of query samples for two settings is fixed.}
    \label{table4}
\end{table}
\section{Conclusion}
In this paper, we proposed a novel Attentive GNN model for few-shot learning.
The proposed Attentive GNN makes full use of the relationships between image samples for knowledge modeling and generalization
By introducing a triple-attention mechanism, Attentive GNN model can effectively alleviate over-smoothing and over-fitting issues when applying deep GNN models.
Extensive experiments are conducted over popular mini-ImageNet and tiered-ImageNet, showing that our proposed Attentive GNN achieving promising results comparing to the state-of-the-art few-shot learning methods. We plan to apply Attentive GNN for other challenging applications in future work.

\setcounter{lemma}{0}
\setcounter{proposition}{0}
\setcounter{theorem}{0}
\section{Appendix}
Here we present the detailed proofs of the results for the proposed Attentive GNN, \ie Lemma~\ref{lemma:self-att} and Theorem~\ref{theo}.

\subsection{Proofs of the results for Attentive GNN}

We prove the main results regarding the proposed attentive GNN. 
First of all, we analyze the proposed node self-attention, whose feature and label vector updates are
\begin{equation} 
\tilde{\mathbf{X}}^{(1)} = \mathbf{C}^{\mathbf{f}} \mathbf{X} \;, \;\;\;\; \mathbf{Y}^{(1)}=\alpha\mathbf{Y} + \left(1-\alpha\right)\mathbf{C}^{\mathbf{f}}\mathbf{Y} \;,
\end{equation}
where $\mathbf{C}^{\mathbf{f}}$ denotes the attention map, $\mathbf{X}$ and $\mathbf{Y}$ (resp. $\mathbf{X}^{(1)}$ and $\mathbf{Y}^{(1)}$) denote the input (resp. output) feature and label vectors, respectively.

We prove Lemma~\ref{lemma:self-att} which shows that the proposed node self-attention can alleviate \textit{Over-fitting} by reducing the model complexity comparing to adding more GNN layer.
The output of the $k$-th GNN layer can be represented as

\begin{equation} 
\mathbf{X}^{(k+1)} = \operatorname{F}_k ( \mathbf{X}^{(k)},\, \mathbf{W}^{(k)} ) = \rho \, ( \hat{\mathbf{A}}^{(k)} \, \mathbf{X}^{(k)} \, \mathbf{W}^{(k)} )
\end{equation}

\begin{lemma} 
The node self-attention module is equivalent to a GNN layer if $\alpha = 0$ as
\begin{equation}
    \mathbf{X}^{(k)} = \begin{bmatrix} \mathbf{X}, \mathbf{Y} \end{bmatrix}\,,\;\; \hat{\mathbf{A}}^{(k)} = \mathbf{C}^{\mathbf{f}}\,,\;\; \mathbf{W}^{(k)} = \mathbf{I}\,,\;\;
\end{equation}
\end{lemma}
\vspace{-0.1in}
\begin{proof}[Proof of Lemma~\ref{lemma:self-att}]
With the condition for equivalence, the output of the $k$-th GNN layer becomes
\begin{equation} \label{eq:selfGNN}
\mathbf{X}^{(k+1)} = \operatorname{F}_k ( \mathbf{X}^{(k)},\, \mathbf{I}) =    \mathbf{C}^{\mathbf{f}} \, \mathbf{X}^{(k)} \, \mathbf{I}  = \mathbf{C}^{\mathbf{f}} \, \mathbf{X}^{(k)}\;\;.
\end{equation}
Thus, (\ref{eq:selfGNN}) is equivalent to putting the node self-attention to replace the $k$-th GNN layer, with $\mathbf{X}^{(k+1)} = \mathbf{X}^{(1)}$ and $\mathbf{X}^{(k)} = \left[ \mathbf{X},\mathbf{Y}\right]$.
\end{proof}

Next, we prove Proposition~\ref{prop:self-att} which shows the model complexity decrease from a trainable GNN layer to the proposed node self-attention module. 
\begin{proposition} 
Applying the node self-attention module to replace a GNN layer in Attentive GNN, reduces the trainable-parameter complexity from $\mathcal{O}\{ d ( d + L) \}$ to $\mathcal{O}\{ 1 \}$, where $L$ denotes the depth of MLP for generating the adjacency metric.
\end{proposition}
\vspace{-0.1in}
\begin{proof}[Proof of Proposition~\ref{prop:self-att}]
For a GNN layer following (\ref{eq1}), both $\mathbf{W}^{(k)}$ and the MLP$^{(k)}$ are trainable, corresponding to free parameters scale as $\mathcal{O}\{d^2\}$ and $\mathcal{O}\{dL\}$, respectively.
On the contrary, based on Lemma~\ref{lemma:self-att}, the proposed node self-attention is equivalent to a GNN layer, with the $\mathbf{W}^{(k)}$ and the MLP$^{(k)}$ fixed. 
The only trainable parameters are the $1 \times 1$ kernels to fuse the $\mathbf{C}^{\mathbf{X}}$ and $\mathbf{C}^{\mathbf{Y}}$, with the complexity scales as $\mathcal{O}\{1\}$.
\vspace{-0.05in}
\end{proof}

Next we show that using graph neighbor attention can help alleviate over-smoothing for training GNN. 
We first quantify the degree of over-smoothing using the definitions from \cite{rong2019dropedge} and \cite{oono2020graph}.
\begin{definition}[Feature Subspace]
Denote the $M$- dimensional subspace 
$\mathcal{M} = \{ \mathbf{U}\Sigma \, | \, \mathbf{U} \in \mathbb{R}^{V \times M}\,,\, \mathbf{U}^T \mathbf{U} = \mathbf{I}_M\,,\,\Sigma \in \mathbf{R}^{M \times d} \}$ as the feature space, with $M \leq V$. 
\end{definition}
\begin{definition}[Projection Loss]
Denote the operator of projection $\mathbf{X} \in \mathbb{R}^{V \times d}$ onto a $M$-dimensional subspace as $\mathbb{P}_M : \mathbb{R}^{V \times d} \rightarrow \mathbb{R}^{V \times d}$ as
\begin{equation}
    \mathbb{P}_{\mathcal{M}} (\mathbf{X}) = \argmin_{\mathbf{Z} \in \mathcal{M}} \left \| \mathbf{X} - \mathbf{Z} \right \|_F\;.
\end{equation}
Denote the projection loss $\theta_{\mathcal{M}} (\mathbf{X})$ as
\begin{equation}
    \theta_{\mathcal{M}} (\mathbf{X}) =  \left \| \mathbf{X} - \mathbb{P}_{\mathcal{M}} (\mathbf{X}) \right \|_F = \min_{\mathbf{Z} \in \mathcal{M}} \left \| \mathbf{X} - \mathbf{Z} \right \|_F\;.
\end{equation}
\end{definition}
\begin{definition}[$\epsilon$-smoothing]
The GNN layer that suffers from  $\epsilon$-smoothing if $\theta_{\mathcal{M}} (\mathbf{X}) < \epsilon$. Given a multi-layer GNN $\bf{G}$ with each the feature output of each layer as $\mathbf{X}^{(k)}$, we define the \textbf{$\epsilon$-smoothing layer} as the minimal value $k$ that encounters $\epsilon$-smoothing, \ie
\begin{equation}
    T(\bf{G}, \epsilon) = \min_{k} \{ \theta_{\mathcal{M}} (\mathbf{X}) < \epsilon \}
\end{equation}
\end{definition}
\begin{definition}[Dimensionality Reduction]
Suppose the dimensionality reduction of the node feature-space after $T$ layers of GNNs is denoted as $\Theta_{T, \bf{G}} = d - T(\bf{G}, \epsilon)$.
\end{definition}

With these definitions from \cite{rong2019dropedge} and \cite{oono2020graph}, we can now prove Theorem~\ref{prop:self-att} for the \textbf{graph neighbor attention} as
\begin{equation} 
\begin{array}{c}
\hat{\mathbf{A}}^{(k)} = \argmin_{\mathbf{A}^{(k)}} \left \| \mathbf{A}^{(k)} - \mathbf{U}^{(k)} \right \|_F \\
s.t. \; \mathbf{U}^{(k)} (i,j) = \mathrm{MLP}^{(k)}\left(\left|\mathbf{x}_{i}^{(k)}-\mathbf{x}_{j}^{(k)}\right|\right)\!, \left\| \mathbf{A}_i^{(k)}\right\|_0 \!\leq\! \beta V.
\end{array}
\end{equation}
Here, $\mathbf{A}_i^{(k)} \in \mathbb{R}^{1 \times V}$ denotes the $i$-th row of $\mathbf{A}^{(k)}$, $\beta \in (0, 1]$ denotes the ratio of nodes maintained for feature update, and $V$ is the number of graph nodes.
Besides, $\mathbf{U}^{(k)}$ is the original adjacency matrix with the graph neighbor attention.

\begin{theorem} 
Denote the same multi-layer GNN model with and without neighbor attention as $\tilde{\bf{G}}$ and $\bf{G}$, respectively.
Besides, denote the number of GNN layers for them to encounter the $\epsilon$-smoothing~\cite{oono2020graph} as $T(\tilde{\bf{G}}, \epsilon)$ and $T(\bf{G}, \epsilon)$, respectively.
With sufficiently small $\beta$ in the node self-attention module, either (i) $T(\tilde{\bf{G}}, \epsilon) \leq T(\bf{G}, \epsilon)$, or (ii) $\Theta_{T(\bf{G}, \epsilon), \bf{G}} > \Theta_{T(\tilde{\bf{G}}, \epsilon), \tilde{\bf{G}}}$, will hold.
\end{theorem}
\begin{proof}[Proof of Theorem~\ref{prop:self-att}]
Given the original $\mathbf{U}^{(k)}$, the solution to (\ref{eq:attSpar}) is achieved using the projection onto a $\ell_0$ unit ball, \ie keeping the $\beta V$ elements of each $\mathbf{U}_i^{(k)}$ with the largest magnitudes~\cite{wen2015structured}, \ie
\begin{equation} \label{sol:balproj}
 \hat{\mathbf{A}}_i^{(k)} (j) = \left \{ \begin{matrix}
 \mathbf{U}_i^{(k)} (j) & , \;\; j \in \Omega_{\beta V}^i \\
 0  & ,\;\; j \in \bar{\Omega}_{\beta V}^i
\end{matrix}\right .
\end{equation}
Here, the set $\Omega_{\beta V}^i = \text{supp}(\hat{\mathbf{A}}_i^{(k)})$ indexes the top-$\beta V$ elements of largest magnitude in $\mathbf{U}_i^{(k)}$, and $\bar{\Omega}_{\beta V}^i$ denotes the complement set of $\Omega_K^i$.
When $\hat{\mathbf{A}}_i^{(k)} (j) = 0$, it is equivalent to remove the edge connecting the $i$-th node and $j$-th node.
Thus, $| \bar{\Omega}_{\beta V}^i |$ equals to the number of edges been dropped by the node self-attention, and $| \bar{\Omega}_{\beta V}^i | \rightarrow V$ as $\beta \rightarrow 0$. 

Therefore, when $\beta$ is sufficiently small, there are sufficient number of edges been dropped by the node self-attention.
Based on the \textbf{Theorem 1} in \cite{rong2019dropedge}, we have either of the two to alleviate over-smoothing phenomenon:
\begin{itemize}
    \item The number of layers without $\epsilon$-smoothing increases by node self-attention, \ie $T(\tilde{\bf{G}}, \epsilon) \leq T(\bf{G}, \epsilon)$.
    \item The information loss (\ie dimensionality reduction by feature embedding) decreases by node self-attention, \ie $\Theta_{T(\bf{G}, \epsilon), \bf{G}} > \Theta_{T(\tilde{\bf{G}}, \epsilon), \tilde{\bf{G}}}$
\end{itemize}
\vspace{-0.05in}
\end{proof}

\bibliography{egbib}

\begin{thebibliography}{46}
\providecommand{\natexlab}[1]{#1}
\providecommand{\url}[1]{\texttt{#1}}
\providecommand{\urlprefix}{URL }
\expandafter\ifx\csname urlstyle\endcsname\relax
  \providecommand{\doi}[1]{doi:\discretionary{}{}{}#1}\else
  \providecommand{\doi}{doi:\discretionary{}{}{}\begingroup
  \urlstyle{rm}\Url}\fi

\bibitem[{Bruna et~al.(2013)Bruna, Zaremba, Szlam, and
  LeCun}]{bruna2013spectral}
Bruna, J.; Zaremba, W.; Szlam, A.; and LeCun, Y. 2013.
\newblock Spectral networks and locally connected networks on graphs.
\newblock \emph{arXiv preprint arXiv:1312.6203} .

\bibitem[{Chen et~al.(2020)Chen, Wang, Liu, Xu, and Darrell}]{chen2020new}
Chen, Y.; Wang, X.; Liu, Z.; Xu, H.; and Darrell, T. 2020.
\newblock A new meta-baseline for few-shot learning.
\newblock \emph{arXiv preprint arXiv:2003.04390} .

\bibitem[{Cheng, Dong, and Lapata(2016)}]{cheng2016long}
Cheng, J.; Dong, L.; and Lapata, M. 2016.
\newblock Long Short-Term Memory-Networks for Machine Reading.
\newblock In \emph{Proceedings of the 2016 Conference on Empirical Methods in
  Natural Language Processing}, 551--561.

\bibitem[{Defferrard, Bresson, and
  Vandergheynst(2016)}]{defferrard2016convolutional}
Defferrard, M.; Bresson, X.; and Vandergheynst, P. 2016.
\newblock Convolutional neural networks on graphs with fast localized spectral
  filtering.
\newblock In \emph{Advances in neural information processing systems},
  3844--3852.

\bibitem[{Fei-Fei, Fergus, and Perona(2006)}]{fei2006one}
Fei-Fei, L.; Fergus, R.; and Perona, P. 2006.
\newblock One-shot learning of object categories.
\newblock \emph{IEEE transactions on pattern analysis and machine intelligence}
  28(4): 594--611.

\bibitem[{Finn, Abbeel, and Levine(2017)}]{finn2017model}
Finn, C.; Abbeel, P.; and Levine, S. 2017.
\newblock Model-agnostic meta-learning for fast adaptation of deep networks.
\newblock In \emph{Proceedings of the 34th International Conference on Machine
  Learning-Volume 70}, 1126--1135. JMLR. org.

\bibitem[{Gidaris and Komodakis(2018)}]{gidaris2018dynamic}
Gidaris, S.; and Komodakis, N. 2018.
\newblock Dynamic few-shot visual learning without forgetting.
\newblock In \emph{Proceedings of the IEEE Conference on Computer Vision and
  Pattern Recognition}, 4367--4375.

\bibitem[{Hao et~al.(2019)Hao, He, Cheng, Wang, Cao, and Tao}]{hao2019collect}
Hao, F.; He, F.; Cheng, J.; Wang, L.; Cao, J.; and Tao, D. 2019.
\newblock Collect and Select: Semantic Alignment Metric Learning for Few-Shot
  Learning.
\newblock In \emph{Proceedings of the IEEE International Conference on Computer
  Vision}, 8460--8469.

\bibitem[{He, Liu, and Hong(2020)}]{he2020memory}
He, J.; Liu, X.; and Hong, R. 2020.
\newblock Memory-Augmented Relation Network for Few-Shot Learning.
\newblock \emph{arXiv preprint arXiv:2005.04414} .

\bibitem[{Hou et~al.(2019)Hou, Chang, Bingpeng, Shan, and Chen}]{hou2019cross}
Hou, R.; Chang, H.; Bingpeng, M.; Shan, S.; and Chen, X. 2019.
\newblock Cross Attention Network for Few-shot Classification.
\newblock In \emph{Advances in Neural Information Processing Systems},
  4005--4016.

\bibitem[{Huang et~al.(2017)Huang, Liu, Van Der~Maaten, and
  Weinberger}]{huang2017densely}
Huang, G.; Liu, Z.; Van Der~Maaten, L.; and Weinberger, K.~Q. 2017.
\newblock Densely connected convolutional networks.
\newblock In \emph{Proceedings of the IEEE conference on computer vision and
  pattern recognition}, 4700--4708.

\bibitem[{Ji et~al.(2020)Ji, Yang, Zhang, and Tay}]{ji2020gfcn}
Ji, F.; Yang, J.; Zhang, Q.; and Tay, W.~P. 2020.
\newblock GFCN: A New Graph Convolutional Network Based on Parallel Flows.
\newblock In \emph{ICASSP 2020-2020 IEEE International Conference on Acoustics,
  Speech and Signal Processing (ICASSP)}, 3332--3336. IEEE.

\bibitem[{Ke et~al.(2020)Ke, Pan, Wen, and Li}]{ke2020compare}
Ke, L.; Pan, M.; Wen, W.; and Li, D. 2020.
\newblock Compare Learning: Bi-Attention Network for Few-Shot Learning.
\newblock In \emph{ICASSP 2020-2020 IEEE International Conference on Acoustics,
  Speech and Signal Processing (ICASSP)}, 2233--2237. IEEE.

\bibitem[{Kim et~al.(2019)Kim, Kim, Kim, and Yoo}]{kim2019edge}
Kim, J.; Kim, T.; Kim, S.; and Yoo, C.~D. 2019.
\newblock Edge-labeling graph neural network for few-shot learning.
\newblock In \emph{Proceedings of the IEEE Conference on Computer Vision and
  Pattern Recognition}, 11--20.

\bibitem[{Kipf and Welling(2017)}]{kipf2016semi}
Kipf, T.~N.; and Welling, M. 2017.
\newblock Semi-supervised classification with graph convolutional networks.
\newblock \emph{International Conference on Learning Representations} .

\bibitem[{Krizhevsky, Sutskever, and Hinton(2012)}]{krizhevsky2012imagenet}
Krizhevsky, A.; Sutskever, I.; and Hinton, G.~E. 2012.
\newblock Imagenet classification with deep convolutional neural networks.
\newblock In \emph{Advances in neural information processing systems},
  1097--1105.

\bibitem[{Lee, Lee, and Kang(2019)}]{lee2019self}
Lee, J.; Lee, I.; and Kang, J. 2019.
\newblock Self-attention graph pooling.
\newblock In \emph{36th International Conference on Machine Learning, ICML
  2019}, 6661--6670. International Machine Learning Society (IMLS).

\bibitem[{Lee et~al.(2019)Lee, Maji, Ravichandran, and Soatto}]{lee2019meta}
Lee, K.; Maji, S.; Ravichandran, A.; and Soatto, S. 2019.
\newblock Meta-learning with differentiable convex optimization.
\newblock In \emph{Proceedings of the IEEE Conference on Computer Vision and
  Pattern Recognition}, 10657--10665.

\bibitem[{Li et~al.(2019{\natexlab{a}})Li, Eigen, Dodge, Zeiler, and
  Wang}]{li2019finding}
Li, H.; Eigen, D.; Dodge, S.; Zeiler, M.; and Wang, X. 2019{\natexlab{a}}.
\newblock Finding task-relevant features for few-shot learning by category
  traversal.
\newblock In \emph{Proceedings of the IEEE Conference on Computer Vision and
  Pattern Recognition}, 1--10.

\bibitem[{Li, Han, and Wu(2018)}]{li2018deeper}
Li, Q.; Han, Z.; and Wu, X.-M. 2018.
\newblock Deeper insights into graph convolutional networks for semi-supervised
  learning.
\newblock In \emph{Thirty-Second AAAI Conference on Artificial Intelligence}.

\bibitem[{Li et~al.(2019{\natexlab{b}})Li, Wang, Xu, Huo, Gao, and
  Luo}]{li2019revisiting}
Li, W.; Wang, L.; Xu, J.; Huo, J.; Gao, Y.; and Luo, J. 2019{\natexlab{b}}.
\newblock Revisiting local descriptor based image-to-class measure for few-shot
  learning.
\newblock In \emph{Proceedings of the IEEE Conference on Computer Vision and
  Pattern Recognition}, 7260--7268.

\bibitem[{Liu et~al.(2019)Liu, Lee, Park, Kim, Yang, Hwang, and
  Yang}]{liu2018learning}
Liu, Y.; Lee, J.; Park, M.; Kim, S.; Yang, E.; Hwang, S.~J.; and Yang, Y. 2019.
\newblock Learning to propagate labels: Transductive propagation network for
  few-shot learning.
\newblock In \emph{7th International Conference on Learning Representations,
  ICLR 2019}. International Conference on Learning Representations, ICLR.

\bibitem[{Liu, Schiele, and Sun(2020)}]{Liu2020E3BM}
Liu, Y.; Schiele, B.; and Sun, Q. 2020.
\newblock An Ensemble of Epoch-wise Empirical Bayes for Few-shot Learning.
\newblock In \emph{European Conference on Computer Vision (ECCV)}.

\bibitem[{Maaten and Hinton(2008)}]{maaten2008visualizing}
Maaten, L. v.~d.; and Hinton, G. 2008.
\newblock Visualizing data using t-SNE.
\newblock \emph{Journal of machine learning research} 9(Nov): 2579--2605.

\bibitem[{Nichol, Achiam, and Schulman(2018)}]{nichol2018first}
Nichol, A.; Achiam, J.; and Schulman, J. 2018.
\newblock On first-order meta-learning algorithms.
\newblock \emph{arXiv preprint arXiv:1803.02999} .

\bibitem[{Oono and Suzuki(2020)}]{oono2020graph}
Oono, K.; and Suzuki, T. 2020.
\newblock Graph neural networks exponentially lose expressive power for node
  classification.
\newblock In \emph{International Conference on Learning Representations}.

\bibitem[{Oreshkin, L{\'o}pez, and Lacoste(2018)}]{oreshkin2018tadam}
Oreshkin, B.; L{\'o}pez, P.~R.; and Lacoste, A. 2018.
\newblock Tadam: Task dependent adaptive metric for improved few-shot learning.
\newblock In \emph{Advances in Neural Information Processing Systems},
  721--731.

\bibitem[{Parikh et~al.(2016)Parikh, T{\"a}ckstr{\"o}m, Das, and
  Uszkoreit}]{parikh2016decomposable}
Parikh, A.; T{\"a}ckstr{\"o}m, O.; Das, D.; and Uszkoreit, J. 2016.
\newblock A Decomposable Attention Model for Natural Language Inference.
\newblock In \emph{Proceedings of the 2016 Conference on Empirical Methods in
  Natural Language Processing}, 2249--2255.

\bibitem[{Qiao et~al.(2019)Qiao, Shi, Li, Wang, Huang, and
  Tian}]{qiao2019transductive}
Qiao, L.; Shi, Y.; Li, J.; Wang, Y.; Huang, T.; and Tian, Y. 2019.
\newblock Transductive episodic-wise adaptive metric for few-shot learning.
\newblock In \emph{Proceedings of the IEEE International Conference on Computer
  Vision}, 3603--3612.

\bibitem[{Ravi and Larochelle(2017)}]{ravi2016optimization}
Ravi, S.; and Larochelle, H. 2017.
\newblock Optimization as a model for few-shot learning.
\newblock \emph{International Conference on Learning Representations} .

\bibitem[{Ren et~al.(2018)Ren, Triantafillou, Ravi, Snell, Swersky, Tenenbaum,
  Larochelle, and Zemel}]{ren2018meta}
Ren, M.; Triantafillou, E.; Ravi, S.; Snell, J.; Swersky, K.; Tenenbaum, J.~B.;
  Larochelle, H.; and Zemel, R.~S. 2018.
\newblock Meta-Learning for Semi-Supervised Few-Shot Classification.
\newblock In \emph{International Conference on Learning Representations}.

\bibitem[{Rong et~al.(2019)Rong, Huang, Xu, and Huang}]{rong2019dropedge}
Rong, Y.; Huang, W.; Xu, T.; and Huang, J. 2019.
\newblock DropEdge: Towards Deep Graph Convolutional Networks on Node
  Classification.
\newblock In \emph{International Conference on Learning Representations}.

\bibitem[{Satorras and Estrach(2018)}]{garcia2017few}
Satorras, V.~G.; and Estrach, J.~B. 2018.
\newblock Few-Shot Learning with Graph Neural Networks.
\newblock In \emph{International Conference on Learning Representations}.

\bibitem[{Snell, Swersky, and Zemel(2017)}]{snell2017prototypical}
Snell, J.; Swersky, K.; and Zemel, R. 2017.
\newblock Prototypical networks for few-shot learning.
\newblock In \emph{Advances in neural information processing systems},
  4077--4087.

\bibitem[{Sperduti and Starita(1997)}]{sperduti1997supervised}
Sperduti, A.; and Starita, A. 1997.
\newblock Supervised neural networks for the classification of structures.
\newblock \emph{IEEE Transactions on Neural Networks} 8(3): 714--735.

\bibitem[{Sun et~al.(2019)Sun, Liu, Chua, and Schiele}]{sun2019meta}
Sun, Q.; Liu, Y.; Chua, T.-S.; and Schiele, B. 2019.
\newblock Meta-transfer learning for few-shot learning.
\newblock In \emph{Proceedings of the IEEE conference on computer vision and
  pattern recognition}, 403--412.

\bibitem[{Sung et~al.(2018)Sung, Yang, Zhang, Xiang, Torr, and
  Hospedales}]{sung2018learning}
Sung, F.; Yang, Y.; Zhang, L.; Xiang, T.; Torr, P.~H.; and Hospedales, T.~M.
  2018.
\newblock Learning to compare: Relation network for few-shot learning.
\newblock In \emph{Proceedings of the IEEE Conference on Computer Vision and
  Pattern Recognition}, 1199--1208.

\bibitem[{Vaswani et~al.(2017)Vaswani, Shazeer, Parmar, Uszkoreit, Jones,
  Gomez, Kaiser, and Polosukhin}]{vaswani2017attention}
Vaswani, A.; Shazeer, N.; Parmar, N.; Uszkoreit, J.; Jones, L.; Gomez, A.~N.;
  Kaiser, {\L}.; and Polosukhin, I. 2017.
\newblock Attention is all you need.
\newblock In \emph{Advances in neural information processing systems},
  5998--6008.

\bibitem[{Veli{\v{c}}kovi{\'c} et~al.(2018)Veli{\v{c}}kovi{\'c}, Cucurull,
  Casanova, Romero, Li{\`o}, and Bengio}]{velivckovic2017graph}
Veli{\v{c}}kovi{\'c}, P.; Cucurull, G.; Casanova, A.; Romero, A.; Li{\`o}, P.;
  and Bengio, Y. 2018.
\newblock Graph Attention Networks.
\newblock In \emph{International Conference on Learning Representations}.

\bibitem[{Vinyals et~al.(2016)Vinyals, Blundell, Lillicrap, Wierstra
  et~al.}]{vinyals2016matching}
Vinyals, O.; Blundell, C.; Lillicrap, T.; Wierstra, D.; et~al. 2016.
\newblock Matching networks for one shot learning.
\newblock In \emph{Advances in neural information processing systems},
  3630--3638.

\bibitem[{Wen, Ravishankar, and Bresler(2015)}]{wen2015structured}
Wen, B.; Ravishankar, S.; and Bresler, Y. 2015.
\newblock Structured overcomplete sparsifying transform learning with
  convergence guarantees and applications.
\newblock \emph{International Journal of Computer Vision} 114(2-3): 137--167.

\bibitem[{Xu et~al.(2018)Xu, Li, Tian, Sonobe, Kawarabayashi, and
  Jegelka}]{xu2018representation}
Xu, K.; Li, C.; Tian, Y.; Sonobe, T.; Kawarabayashi, K.-i.; and Jegelka, S.
  2018.
\newblock Representation Learning on Graphs with Jumping Knowledge Networks.
\newblock In \emph{ICML}.

\bibitem[{Yang et~al.(2020)Yang, Li, Zhang, Zhou, Zhou, and Liu}]{yang2020dpgn}
Yang, L.; Li, L.; Zhang, Z.; Zhou, X.; Zhou, E.; and Liu, Y. 2020.
\newblock DPGN: Distribution Propagation Graph Network for Few-shot Learning.
\newblock In \emph{Proceedings of the IEEE/CVF Conference on Computer Vision
  and Pattern Recognition}, 13390--13399.

\bibitem[{Ye et~al.(2020)Ye, Hu, Zhan, and Sha}]{ye2020few}
Ye, H.-J.; Hu, H.; Zhan, D.-C.; and Sha, F. 2020.
\newblock Few-shot learning via embedding adaptation with set-to-set functions.
\newblock In \emph{Proceedings of the IEEE/CVF Conference on Computer Vision
  and Pattern Recognition}, 8808--8817.

\bibitem[{Zhang et~al.(2020{\natexlab{a}})Zhang, Cai, Lin, and
  Shen}]{Zhang_2020_CVPR}
Zhang, C.; Cai, Y.; Lin, G.; and Shen, C. 2020{\natexlab{a}}.
\newblock DeepEMD: Few-Shot Image Classification With Differentiable Earth
  Mover's Distance and Structured Classifiers.
\newblock In \emph{IEEE/CVF Conference on Computer Vision and Pattern
  Recognition (CVPR)}.

\bibitem[{Zhang et~al.(2020{\natexlab{b}})Zhang, Zhang, Lu, Xiang, and
  Wen}]{zhang2020adargcn}
Zhang, J.; Zhang, M.; Lu, Z.; Xiang, T.; and Wen, J. 2020{\natexlab{b}}.
\newblock AdarGCN: Adaptive Aggregation GCN for Few-Shot Learning.
\newblock \emph{arXiv preprint arXiv:2002.12641} .

\end{thebibliography}
\end{document}